\documentclass[letterpaper, 10 pt, conference]{ieeeconf}  
\IEEEoverridecommandlockouts                              
\usepackage{graphics} 
\usepackage{epsfig} 
\usepackage{mathptmx} 
\usepackage{times} 
\usepackage{amsmath} 
\usepackage{amssymb}  
\usepackage{psfrag} 
\usepackage{subcaption,tikz} 
\usepackage{url}
\usepackage[english]{babel}
\usepackage{algorithm,algorithmic}
\newtheorem{theorem}{Theorem}[section]

\newtheorem{lemma}[theorem]{Lemma}
\newtheorem{proposition}[theorem]{Proposition}

\newcommand{\CASE}[1]{\STATE \textbf{case} #1\textbf{:} \begin{ALC@g}}
\newcommand{\ENDCASE}{\end{ALC@g}}

\newcommand{\DEFAULT}{\STATE \textbf{default:} \begin{ALC@g}}
\newcommand{\ENDDEFAULT}{\end{ALC@g}}
\newcommand{\DEFAULTLINEcompareCLFCBF}[1]{\STATE \textbf{default:} }

\title{\LARGE \bf
Safe Learning of Quadrotor Dynamics Using Barrier Certificates *
}

\author{Li Wang, Evangelos A. Theodorou, and Magnus Egerstedt$^\dagger$
\thanks{*The work by the first and third authors was sponsored by Grant No.
1544332 from the U.S. National Science Foundation.}
\thanks{$^\dagger$Li Wang and Magnus Egerstedt are with the School of Electrical and Computer Engineering, Evangelos A. Theodorou is with the School of Aerospace Engineering, Georgia Institute of Technology, Atlanta, GA 30332, USA.  Email: {\tt\small \{liwang, magnus, evangelos.theodorou\}@gatech.edu} }
}

\begin{document}
\maketitle
\thispagestyle{empty}
\pagestyle{empty}

\begin{abstract}
To effectively control complex dynamical systems, accurate nonlinear models are typically needed. However, these models are not always known. In this paper, we present a data-driven approach based on Gaussian processes that learns models of quadrotors operating in partially unknown environments. What makes this challenging is that if the learning process is not carefully controlled, the system will go unstable, i.e., the quadcopter will crash. To this end, barrier certificates are employed for safe learning. The barrier certificates establish a non-conservative forward invariant safe region, in which high probability safety guarantees are provided based on the statistics of the Gaussian Process. A learning controller is designed to efficiently explore those uncertain states and expand the barrier certified safe region based on an adaptive sampling scheme. In addition, a recursive Gaussian Process prediction method is developed to learn the complex quadrotor dynamics in real-time. Simulation results are provided to demonstrate the effectiveness of the proposed approach.
\end{abstract}

\section{Introduction}
Safety is crucial to many physical control dynamical systems, such as autonomous vehicles, industrial robots, chemical reactors, and air-traffic control systems \cite{hovakimyan2011mathcal, ames2017control, berkenkamp2016safe}. If the system reaches certain unsafe states or even fails, both the operator and the controlled plant might be put in serious danger. The existence of model inaccuracies and unknown disturbances create an even greater challenge to the design of safe controllers for these systems.

Tools such as robust control and adaptive control methods have been developed in classic control theory to ensure the safety and stability of the system, see \cite{fukushima2007adaptive,bemporad1999robust} and the references therein. Meanwhile, machine learning based control approaches are becoming increasingly popular as a way to deal with inaccurate models \cite{deisenroth2015gaussian, pan2014probabilistic}, due to their abilities to infer unknown models from data and actively improve the performance of the controller with the learned model. In contrast to classic control methods, learning based control approaches require only limited expert knowledge and fewer assumptions about the system \cite{vinogradska2016stability}. However, there always exists an inherent trade-off between safety and performance in these methods \cite{aswani2013provably}. Data-driven learning approaches rarely provides safety guarantees, which limits their applicability to real-world safety critical control dynamical systems \cite{vinogradska2016stability}. The objective of this paper is to construct high probability safety guarantees for Gaussian Process (GP) based learning approaches using barrier certificates.

In order to promote the application of learning based control methods in safety-critical systems, a number of safe learning approaches have been proposed in the literature. Among these methods, the use of learning Control Lyapunov Functions (CLF) is shown to be a promising approach. A learning from demonstration method was developed in \cite{khansari2014learning} to search for a CLF from several demonstrations, and the learned CLF was used to stabilize the system. But the learned controller did not consider actuator limits and other safe operation constraints. \cite{ravanbakhsh2017learning} introduced a verifier to explicitly validate the learned CLF. However, when the model of the system is inaccurate, the verifier needs to check an infinite number of inequalities throughout the state space, which is computationally difficult \cite{ito2017second}. \cite{berkenkamp2017safe} seeks to learn CLF and maximize the safe operation region for the system with GP model. High probability safety guarantees are provided based on Lyapunov stability and GP statistics. In addition, a reachability-based safe learning approach was presented in \cite{akametalu2014reachability} to reduce the conservativeness of reachability analysis by learning the disturbance from data.

In contrast to the aforementioned methods, this paper interprets the safe operation region as general invariant sets established with barrier certificates, which permits a much richer set of safe control options, rather than Lyapunov sublevel sets. The barrier certificates formally define a forward invariant safe region, where all system trajectories starting in this region remains in this region for all time \cite{prajna2004stochastic, Xu2015Robustness, ames2017control}. With the barrier certificates, the safety of the system can be certified without explicitly computing the forward reachable set \cite{prajna2007framework}. Barrier certificates were successfully applied to many safety critical dynamical systems, such as adaptive cruise control \cite{ames2017control}, bipedal walking \cite{nguyen20163d}, quadrotor control \cite{wu2016safety}, and swarm robotics \cite{wang2016multiobj, wang2017multidrone}. In this paper, we construct a safe operation region with barrier certificates, and gradually expand the certified safe region as the uncertainty of the system reduces. The unknown dynamics of the system is represented with a GP model, which provides both the mean and variance of the prediction. Using the statistics of GP model, a high probability safety guarantee of the system with inaccurate model is provided. 

The search for maximum volume barrier certificates involves the validation of an infinite number of inequality constraints, which is computationally expensive. Inspired by the discrete sampling technique used in \cite{berkenkamp2016safe}, we design an adaptive sampling algorithm to significantly reduce the computation intensity, i.e., the more certain regions in the state space are sampled less without loss of safety guarantees. In addition, a recursive learning strategy based on GP is designed to learn the complex 3D nonlinear quadrotor dynamics online. The learned dynamical model of the quadrotor is then incorporated into a differential flatness based flight controller to improve the trajectory tracking performance. 

The main contributions of this paper are threefold. First, a safe learning strategy is developed based on barrier certificates, which admits a rich set of learning control options. Second, an adaptive sampling algorithm is proposed to significantly reduce the computation intensity of the learning process. Third, an recursive learning strategy based on GP is presented to learn the complex 3D nonlinear quadrotor dynamics online.

The rest of this paper is organized as follows. The preliminaries of barrier certificates and GP are briefly revisited in Section \ref{sec:prelimBarrier}. A safe learning strategy based on barrier certificates is presented in Section \ref{sec:learncbf}. Section \ref{sec:quadlearn} contains a real-time learning algorithm for 3D quadrotor dynamics based on GP. Simulation results are provided in Section \ref{sec:sim}, and the paper is ended by conclusions in Section \ref{sec:conclusions}.

\section{Preliminaries of Barrier Certificates and Gaussian Process} \label{sec:prelimBarrier}
Preliminary results regarding the two fundamental tools, i.e., barrier certificates and Gaussian Process, used to formulate the safe learning strategy are presented in this section.

\subsection{Barrier Certificates and Set Invariance}
Consider a control affine dynamical system
\begin{equation}\label{eqn:sysaffine}
\dot{x} =f(x)+ g(x)u,
\end{equation}
where $x\in\mathcal{X}\subseteq\mathbb{R}^n$ and $u\in \mathcal{U}\subseteq\mathbb{R}^m$ are the state and control of the system,  $f:\mathbb{R}^n\rightarrow\mathbb{R}^n$ and $g:\mathbb{R}^n\rightarrow\mathbb{R}^m$ are Lipschitz continuous. Let the safe set of the system be encoded as the superlevel set of a smooth function $h:\mathbb{R}^n\rightarrow\mathbb{R}$,
\begin{equation}
\mathcal{C} = \{x\in\mathbb{R}^n\ |\ h(x)\geq 0\}.
\label{eqn:admSet}
\end{equation}
The function $h(x)$ is termed a Control Barrier Function (CBF), if there exists an extended class-$\kappa$ function ($\kappa(0)=0$ and strictly increasing) such that
$$
\sup_{u\in \mathcal U} \left\{\frac{\partial h}{\partial x}f(x)+\frac{\partial h}{\partial x}g(x)u+\kappa(h(x))\right\}\geq 0,$$
for all $x\in\mathcal{E}$ with $\mathcal{C}\subseteq\mathcal{E}$. 
 
Given a CBF, the barrier certified safe control space $S(x)$ is defined as
\begin{align}
\label{eqn:reqB2}
S (x) =
     \left \lbrace u\in U\ \, | \ \frac{\partial h}{\partial x}f(x)+\frac{\partial h}{\partial x}g(x)u+\kappa(h(x))\geq 0\right\rbrace ,~x\in \mathcal{E}. \nonumber
\end{align}
With barrier certificates, the invariance property of $\mathcal{C}$ is established with the following theorem,

\noindent\textbf{Theorem \cite{Xu2015Robustness}:} \textit{Given a set $\mathcal{C} \subset \mathbb{R}^n$ defined by (\ref{eqn:admSet}) and a CBF $h$ defined on $\mathcal{E}$, with $\mathcal{C}\subseteq\mathcal{E} \subset\mathbb{R}^n$, any Lipschitz continuous controller $u\colon\mathcal{E}\to\mathbb{R}$ such that $u\in S(x)$ for the system (\ref{eqn:sysaffine}) renders the set $\mathcal{C}$ forward invariant. }

This type of barrier certificates expands the certified safe control space significantly by allowing $h(x)$ to decrease within $\mathcal{C}$ as opposed to strictly increasing \cite{Xu2015Robustness, ames2017control}. Compared with Lyapunov sublevel set based safe region, barrier certificates provide a more permissive notion of safety. As a result, barrier certificates based safe learning controllers have more freedom to efficiently explore those unknown states. This fact can be illustrated with the following example.  

\textit{Example $1$}: Consider an autonomous dynamical system
 \begin{equation}\label{eqn::sysauto}
\begin{bmatrix} \dot{x}_1 \\ \dot{x}_2 \end{bmatrix} 
= \begin{bmatrix} x_2+0.8x_2^2 \\ -x_1-x_2+x_1^2x_2  \end{bmatrix},
\end{equation}
the safe region of this system is estimated with both the Lyapunov sublevel set and barrier certificates. 

Since (\ref{eqn::sysauto}) is a polynominal system, the safe sets can be computed directly with Sum-of-Squares programs using YALMIP \cite{lofberg2005yalmip} and SMRSOFT \cite{chesi2011domain} solvers. Both the Lyapunov function and barrier certificates are limited to second order polynomials. The safe region estimated with the optimal polynomial Lyapunov function is $$\mathcal{A}_1=\{x~|~V^*(x)\leq 1\},$$ where $V^*(x)=1.343x_1^2+0.5155x_1x_2+1.152x_2^2$.

The safe region estimated with barrier certificates is $$\mathcal{A}_2=\{x~|~h^*(x)\geq 0\},$$
where $h^*(x) = 1 -0.4254x_1 -0.3248x_2 -0.7549x_2^2-0.8616x_1^2-0.2846x_1x_2$. 
\begin{figure}[h]
  \centering
  \includegraphics[width=0.9\linewidth]{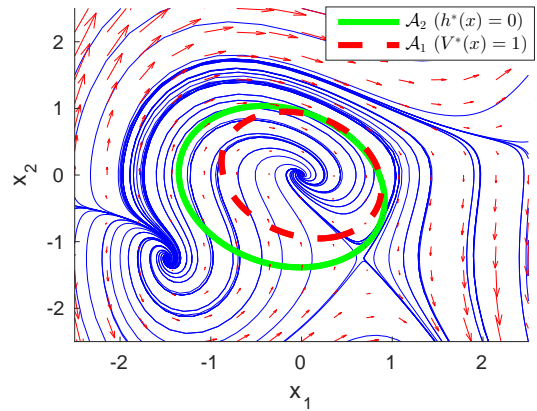}
  \captionof{figure}{Estimates of safe regions for system (\ref{eqn::sysauto}). The regions enclosed by the dashed red ellipse and solid green ellipse are estimated safe regions with optimal polynomial Lyapunov function $V^*(x)$ and barrier certificates $h^*(x)$, respectively.}
  \label{fig:LF_BF}
\end{figure}

From Fig. \ref{fig:LF_BF}, it can be observed that the barrier certified safe region $\mathcal{A}_2$ is much larger than the Lyapunov based safe region $\mathcal{A}_1$. Consequently, safe learning controller based on barrier certificates are allowed to explore more states of the system. In this paper, we will leverage the non-conservative safety guarantee of barrier certificates to allow a much richer set of safe learning control options.



\subsection{Gaussian Processes}
A GP is a nonparametric regression method that can capture complex unknown functions \cite{rasmussen2006gaussian}. With a GP, every point in the state space is associated with a normally distributed random variable, which allows us to derive high probability statements about the system.

Adding some unknown dynamics $d(x)$ to the original class of control-affine systems (\ref{eqn:sysaffine}),  we now consider a system with partially unknown dynamics in this paper, i.e.,
\begin{equation}
\dot{x} = f(x) + g(x)u +{d(x)},
\end{equation}
where $x \in \mathcal{X} \subseteq\mathbb{R}^n$ and $u\in \mathcal{U} \subseteq\mathbb{R}^m$ are the state and control of the system. Although the proposed method applies to general dynamical systems, here we restrict our attention to the class of systems that can be addressed with existing computation tools. It is also assumed that $d(x)$ is Lipschitz continuous. This assumption is necessary, because we want to generalize the learned dynamics to states that are not explored before.

Since the unmodeled dynamics $d(x)$ is $n$ dimensional, each dimension is approximated with a GP model $\mathcal{GP}(0,k(x,x'))$ with a prior mean of zero and a covariance function of $k(x,x')$, where $k(x,x')$ is the kernel function to measure the similarity between any two states $x,x'\in\mathcal{X}$. In order to make GP inferences on the unknown dynamics, we need to get measurements of $d(x)$. This measurement $\hat{d}(x)$ is obtained indirectly by subtracting the inaccurate model prediction $[f(x)+g(x)u]$ from the noisy measurement of the system dynamics $[\dot{x}+\mathcal{N}(0, \sigma_n^2)]$. Since any finite number of data points form a multivariate normal distribution, we can obtain the posterior distribution of $d(x_*)$ at any query state $x_*\in\mathcal{X}$ by conditioning on the past measurements \cite{rasmussen2006gaussian}. 

Given a collection of $w$ measurements $y_w = [\hat{d}(x_1), \hat{d}(x_2), ..., \hat{d}(x_w)]^T$, the mean $m(x_*)$ and variance $\sigma^2(x_*)$ of $d(x_*)$ at the query state $x_*$ are
\begin{eqnarray}\label{eqn:GPinfer}
m(x_*) &=& k_*^T(K+\sigma_n^2I)^{-1}y_w, \\
\sigma^2(x_*) &=& k(x_*,x_*) - k_*^T(K+\sigma_n^2I)^{-1}k_*, 
\end{eqnarray}
where $\lfloor K\rfloor_{(i,j)}=k(x_i,x_j)$ is the kernel matrix, and $k_* = [k(x_1, x_*), k(x_2, x_*), ..., k(x_w, x_*)]^T$.

With the learned system dynamics based on GP, a high probability confidence interval of the unmodeled dynamics $d(x)$ can be established as
\begin{equation} \label{eqn:setD}
\mathcal{D}(x) = \{{d}~|~m(x)-k_\delta\sigma(x)\leq{d}\leq m(x)+ k_\delta\sigma(x)\},
\end{equation}
where $k_\delta$ is a design parameter to get $(1-\delta)$ confidence, $\delta\in (0,1)$. For instance, $95.5\%$ and $99.7\%$ confidence are achieved at $k_\delta=2$ and $k_\delta=3$, respectively. 


\section{Safe Learning with Barrier Certificates}\label{sec:learncbf}
In order to ensure that the learning based controller never enters the unsafe region, we will learn barrier certificates for the system and use the learned certificates to regulate the controller. As discussed in Section \ref{sec:prelimBarrier}, the barrier certificates certify a safe region that is forward invariant. We can first start with an conservative barrier certificate with certified safe region $\mathcal{C}_0(x)$, then gradually expand this certified safe region with the collected data until it stops growing. This incremental learning process is visualized in Fig. \ref{fig:cbf}.
\begin{figure}[h]
\centering
  \includegraphics[width=.85\linewidth]{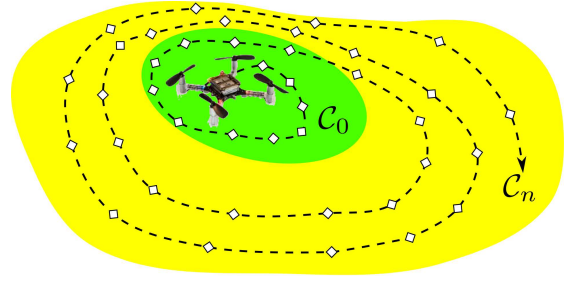}
  \captionof{figure}{Incremental learning of the barrier certificates. The green region $\mathcal{C}_0$ and the yellow regions $\mathcal{C}_n$ are the initial and final barrier certified safe regions, respectively. The barrier certified safe region gradually grows as more and more data points are sampled in the state space.}
  \label{fig:cbf}
\end{figure}

More concretely, the goal of the learning process is to maximize the volume of the barrier certified safe region $\mathcal{C}$ by adjusting $h(x)$, i.e.,
\begin{equation*}
\label{eqn:maxBarrierC} 
 \begin{aligned}
&  & &\underset{h(x)}{\text{max}}\quad
  \text{vol}(\mathcal{C})   \\ 
 &  \text{s.t.}
 & & \underset{u\in\mathcal{U}}{\text{max}} \underset{{d}\in\mathcal{D}(x)}{\text{min}}  \left\{\frac{\partial h}{\partial x}(f(x)+g(x){u}+{d})+\gamma h(x)\right\}\geq 0, \\
&  & & \qquad \qquad\qquad\qquad\qquad\qquad\qquad\qquad\qquad   \forall x \in\mathcal{C}. 
 \end{aligned}
\end{equation*}
Since $u$ and $d$ are independent from each other, we can rewrite this optimization problem into 
\begin{equation}
\label{eqn:maxBarrierCisolate}
 \begin{aligned}
&  & &\underset{h(x)}{\text{max}}\quad
  \text{vol}(\mathcal{C})   \\ 
 &  \text{s.t.}
 & & \underset{u\in\mathcal{U}}{\text{max}}  \left\{\frac{\partial h_k}{\partial x}g(x){u}\right\} + \underset{{d}\in\mathcal{D}(x)}{\text{min}}  \left\{\frac{\partial h}{\partial x}{d}\right\}  \\   &  & & \qquad\qquad\qquad + \frac{\partial h}{\partial x}f(x)+\gamma h(x) \geq 0, 
  \forall x \in\mathcal{C}
 \end{aligned}
\end{equation}
Using the high confidence interval $\mathcal{D}(x)$ in (\ref{eqn:setD}), the barrier certificates constraint can be considered as
\begin{equation}
\label{eqn:maxBarrierCsetD}
 \begin{aligned}
&  & &\underset{h(x)}{\text{max}}\quad
  \text{vol}(\mathcal{C})   \\ 
 &  \text{s.t.}
 & & \underset{u\in\mathcal{U}}{\text{max}}  \left\{\frac{\partial h}{\partial x}g(x){u}\right\}  
+\frac{\partial h}{\partial x}m(x) -k_\delta\left|\frac{\partial h}{\partial x}\right|\sigma(x)
  \\   &  & & \qquad\qquad\qquad + \frac{\partial h}{\partial x}f(x)+\gamma h(x) \geq 0, 
  \forall x \in\mathcal{C}.
 \end{aligned}
\end{equation}
When more data points are collected about the system dynamics, the uncertainty $\sigma(x)$ will gradually decrease. As a result, more states will satisfy the barrier certificates constraint. The goal of the exploration task is to actively collect data to reduce $\sigma(x)$ and maximize the volume of $\mathcal{C}$.

It should be pointed out that the barrier certified region maximization problem (\ref{eqn:maxBarrierCsetD}) is a non-convex, infinite dimensional optimization problem, which is intractable to solve in practice. We will make two simplifications to make it solvable, namely by employing adaptive sampling of the state space and parameterization of the shape of $\mathcal{C}$.

\subsection{Adaptive Sampling of the State Space} \label{sec:adaptsample}
Due to the Lipschitz continuity of the system dynamics, the safety of the system in $\mathcal{X}$ can be evaluated by only sampling a finite number of points in $\mathcal{X}$. Inspired by \cite{berkenkamp2016safe}, we will show that we can adaptively sample the state space without losing safety guarantees. Similar to {\it Lemma 4} in \cite{berkenkamp2016safe}, it can be shown that $h(x)$ and $\dot{h}(x)$ are Lipschitz continuous in $x$ with Lipschitz constants $L_{{h}}$ and $L_{\dot{h}}$, respectively.

Let $\mathcal{X}_\tau\subset \mathcal{X}$ be a discretization of the state space $\mathcal{X}$. The closest point in $\mathcal{X}_\tau$ to $x\in\mathcal{X}$ is denoted as $[x]_\tau$, where $\|x-[x]_\tau\|\leq \frac{\tau}{2}$. 
\begin{lemma} \label{lm:discretize}
If the following condition holds for all $x\in\mathcal{X}_\tau$, 
\begin{eqnarray}\label{eqn:xtauCondition}
\underset{u\in\mathcal{U}}{\text{max}}  \left\{\frac{\partial h}{\partial x}g(x){u}\right\} 
+\frac{\partial h}{\partial x}m(x) -k_\delta\left|\frac{\partial h}{\partial x}\right|\sigma(x) \nonumber \\
 + \frac{\partial h}{\partial x}f(x)+\gamma h(x) \geq (L_{\dot{h}} +\gamma L_{{h}})\tau,
\end{eqnarray}
then the safety barrier constraint
\begin{equation}\label{eqn:barrierConstraint}
\underset{u\in\mathcal{U}}{\text{max}} \underset{d\in\mathcal{D}(x)}{\text{min}} \left\{\frac{\partial h}{\partial x}(f(x)+g(x){u}+d)+\gamma h(x)\right\}   \geq 0
\end{equation}
is satisfied for all $x\in\mathcal{X}$ with probability $(1-\delta)$, $\delta \in(0,1)$.
\end{lemma}
\begin{proof}
With the definition of the high confidence interval $\mathcal{D}(x)$, (\ref{eqn:xtauCondition}) can be rewritten as
$$\underset{u\in\mathcal{U}}{\text{max}} \underset{{d}\in\mathcal{D}(x)}{\text{min}}  \left\{\frac{\partial h}{\partial x}(f(x)+g(x){u}+{d})+\gamma h(x)\right\}\geq (L_{\dot{h}} +\gamma L_{{h}})\tau,$$
with a probability of $(1-\delta)$, for all $x\in\mathcal{X}_\tau$. This is equivalent to
$$\dot{h}(x)+\gamma h(x)\geq (L_{\dot{h}} +\gamma L_{{h}})\tau,$$
for all $x\in\mathcal{X}_\tau$.

Because of the Lipschitz continuity of $h(x)$ and $\dot{h}(x)$, we have for any $x\in\mathcal{X}$,
\begin{eqnarray*}
\dot{h}(x)+\gamma h(x) &\geq& (\dot{h}([x]_\tau) - L_{\dot{h}}\tau) + \gamma (h([x]_\tau) - L_{{h}}\tau)  \\
&\geq& 0.
\end{eqnarray*}
This means that the safety barrier constraint is satisfied for any $x\in\mathcal{X}$, if (\ref{eqn:xtauCondition}) holds for all $x\in\mathcal{X}_\tau$. 
\end{proof}

With the discretization of the state space, we only need to sample a finite number of points to validate the barrier certificates. However, the number of required sampling points is still very large. The following adaptive sampling strategy further reduces the number of sampling points required.
\begin{proposition}
If the following condition is satisfied at $x\in\mathcal{X}$,
\begin{eqnarray}\label{eqn:xktauCondition}
\underset{u\in\mathcal{U}}{\text{max}}  \left\{\frac{\partial h}{\partial x}g(x){u}\right\} 
+\frac{\partial h}{\partial x}m(x) -k_\delta\left|\frac{\partial h}{\partial x}\right|\sigma(x) \nonumber \\
 + \frac{\partial h}{\partial x}f(x)+\gamma h(x) \geq (L_{\dot{h}} +\gamma L_{{h}})k_\tau\tau,
\end{eqnarray}
with $k_\tau\geq 0$, then the safety barrier constraint (\ref{eqn:barrierConstraint}) is satisfied for all $y\in\mathcal{X}$ such that $\|x-y\|\leq k_\tau\tau$.
\end{proposition}
\begin{proof}
The proof is similar to {\it lemma} \ref{lm:discretize}.
\end{proof}
Leveraging the Lipschitz continuity of the barrier certificates, we can adaptively sample the state space without losing safety guarantees. Sparse sampling is performed at places with large safety margin, while dense sampling is only required at places with small safety margin.

\subsection{Parameterization of the Barrier Certificates} 
Because maximizing the volume of $\mathcal{C}$ is a non-convex problem in general, we can parameterize the barrier certificate $h_\mu(x)$ with $\mu$ to simplify the optimization problem. For example, $h_\mu(x)$ can be formulated as $1-Z(x)^T\mu Z(x)$, where $Z(x)$ is the vector of monomials, and $\mu$ is a positive semi-definite matrix. Then maximizing $\text{vol}(\mathcal{C})$ is equivalent to minimize the trace of $\mu$. Further simplification can be made to fix the shape of $\mathcal{C}$ (by optimizing only with the known dynamics) and enlarge the level set of barrier certificates.

With the shape parameterization and adaptive sampling technique, the barrier certificate maximization problem (\ref{eqn:maxBarrierCsetD}) can be written as 
\begin{equation}
\label{eqn:maxBarrierParam}
 \begin{aligned}
&  & &\underset{\mu}{\text{max}}\quad
  \text{vol}(\mathcal{C})   \\ 
 &  \text{s.t.}
 & & \underset{u\in\mathcal{U}}{\text{max}}  \left\{\frac{\partial h_\mu}{\partial x}g(x){u}\right\}  
+\frac{\partial h_\mu}{\partial x}m(x) -k_\delta\left|\frac{\partial h_\mu}{\partial x}\right|\sigma(x)
  \\   &  & & \qquad + \frac{\partial h_\mu}{\partial x}f(x)+\gamma h_\mu(x) \geq (L_{\dot{h}} +\gamma L_{{h}})\tau, 
  \forall x \in\mathcal{C}\cap\mathcal{X}_\tau.
 \end{aligned}
\end{equation}

In order to increase the learning efficiency during the exploration phase, the most uncertain state in $\mathcal{C}$ is sampled,
\begin{equation}\label{eqn:xnext}
x_\text{next} = \underset{x\in\mathcal{C}\cap\mathcal{X}_\tau}{\text{argmax}}\quad \sigma(x). 
\end{equation}
It is assumed that a nominal exploration controller $\hat{u}$ can always be designed to drive the system from the current state $x$ to $x_\text{next}$, i.e., $\hat{u} = GoTo(x, x_\text{next})$. Then the safety barrier certificates are enforced through a QP-based controller to ``rectify" the nominal control such that the system is always safe,
\begin{equation}
\label{eqn:QPcontroller}
 \begin{aligned}
{u}^* =  & \:\: \underset{u\in\mathcal{U}}{\text{argmin}}
 & & J(u) = \|{u} - \hat{u} \|^2    \\ 
 & \quad \text{s.t.}
 & & \frac{\partial h}{\partial x}g(x){u}
+\frac{\partial h}{\partial x}m(x) -k_\delta\left|\frac{\partial h}{\partial x}\right|\sigma(x)\\
 &&& \qquad\qquad\qquad + \frac{\partial h}{\partial x}f(x)+\gamma h(x) \geq 0.
 \end{aligned}
\end{equation}
Therefore, the actual exploration controller $u^*$ tries to stay as close as possible to the desired controller $\hat{u}$, while always honoring the safety requirements. The exploration phase ends when the safe region $\mathcal{C}$ does not grow any more. The learned maximum barrier certificates can be further used to regulate other control tasks the system want to achieve.

\subsection{Overview of the Safe Learning Algorithm}
An overview of the barrier certificates based safe learning algorithm is provided in {\bf Algorithm 1}. At the beginning, a conservative barrier certified safe region $\mathcal{C}_0$ is provided. The most uncertain state $x_\text{next}$ is computed based on the current GP model. Then, the QP based controller (\ref{eqn:QPcontroller}) is used to ensure that the system is driven to $x_\text{next}$ without ever leaving $\mathcal{C}_n$. After updating the GP model with the sampled data at $x_\text{next}$, the barrier certificate optimization problem (\ref{eqn:maxBarrierParam}) is solved. The adaptive sampling technique (\ref{eqn:xktauCondition}) is adopted here to reduce the number of states to be sampled. This process is repeated until the safe region $\mathcal{C}_n$ stops growing.

\begin{algorithm}
 \caption{Barrier Certificates based Safe Learning}
 \begin{algorithmic}[1]
 \renewcommand{\algorithmicrequire}{\textbf{Input:}}
 \renewcommand{\algorithmicensure}{\textbf{Output:}}
 \REQUIRE Initial safe set $\mathcal{C}_0\subseteq \mathcal{X}$, GP model $\mathcal{GP}(0, k(x,x'))$, discretization $\mathcal{X}_\tau$, tolerance $\epsilon$
 \ENSURE  Final safe set $\mathcal{C}_n$
 \\ \textit{Initialization} : $n = 0, x=x_0$
  \REPEAT
  \STATE $n = n+1$
  \STATE Find $x_\text{next}$ with (\ref{eqn:xnext})
  \STATE Design nominal controller $\hat{u} = GoTo(x, x_\text{next})$ 
  \STATE Rectify $\hat{u}$ with (\ref{eqn:QPcontroller}) and drive to $x_\text{next}$
  \STATE Sample $x_\text{next}$, update GP
  \STATE Expand vol($\mathcal{C}_n$) with (\ref{eqn:maxBarrierParam}) and adaptive sampling (\ref{eqn:xktauCondition}) 
  \UNTIL{vol($\mathcal{C}_n$)-vol($\mathcal{C}_{n-1}$)$\leq \epsilon$}
 \RETURN $\mathcal{C}_n$
 \end{algorithmic} 
 \end{algorithm}

\section{Online Learning of Quadrotor Dynamics}\label{sec:quadlearn}
The safe learning approach developed in Section \ref{sec:learncbf} relies on a learning controller that drives the system to explore interested states. The challenge of designing this learning controller is that the 3D quadrotor system considered in this paper is highly nonlinear and unstable. In this section, we will present a recursive learning controller based on GP to learn the complex quadrotor dynamics online.
\subsection{Differential Flatness of 3D Quadrotor Dynamics}
The quadrotor is a well-modelled dynamical system with forces and torques generated by four propellers and gravity \cite{zhou2014vector}. The relevant coordinate frames and Euler angles (roll $\phi$, pitch $\theta$, and yaw $\psi$) are illustrated in Fig. \ref{fig:coord}. The world, body, and intermediate frames (after yaw angle rotation) are denoted by the subscripts $w$, $b$, and $c$, respectively. 
\begin{figure}[h]
  \centering
  \includegraphics[width=.6\linewidth]{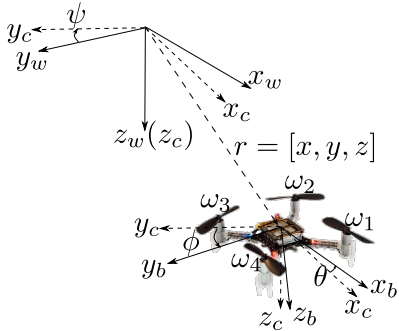}
  \captionof{figure}{Quadrotor coordinate frames.}
  \label{fig:coord}
\end{figure}

The Euler angles are defined with the $ZYX$ convention. Hence, the rotation matrix from the body frame to the world frame can be written as
\begin{equation*}
R=\begin{bmatrix}
c\theta c\psi & s\phi s\theta c\psi - c\phi s\psi  & c\phi s\theta c\psi +s\phi s\psi \\
c\theta s\psi & s\phi s\theta s\psi + c\phi c\psi  & c\phi s\theta s\psi -s\phi c\psi \\
-s\theta & s\phi c\theta  & c\phi c\theta 
\end{bmatrix},
\end{equation*}
where $s\theta$ and $c\theta$ stand for $\sin\theta$ and $\cos\theta$, respectively. 

Here, we adopt the quadrotor model used in \cite{hehn2015real} to describe the nonlinear quadrotor dynamics,  
\begin{equation}\label{eqn:quaddyn}
  \begin{cases}
\ddot{r} &= \quad gz_w + \frac{1}{m}Rz_wf_z,  \\
\begin{bmatrix} \dot{\phi}\\ \dot{\theta} \\ \dot{\psi}
\end{bmatrix} &= \quad \begin{bmatrix} 1 & s\phi t\theta & c\phi t\theta \\
0 & c\phi & -s\phi  \\ 0&s\phi sc\theta & c\phi sc\theta
\end{bmatrix} \omega,
  \end{cases}
\end{equation}
where $z_w=[0~0~1]^T$, and $r=[x,y,z]^T$, $m$, and $g$ are the position of the center of mass, the mass, and the gravitational acceleration  of the quadrotor, respectively. $t\theta$ and $sc\theta$ are short for $\tan\theta$ and $\sec\theta$. The control inputs of the quadrotor are the body rotational rates ($\omega=[\omega_x, \omega_y, \omega_z]^T$) and the thrust ($f$). It is assumed that the body rotational rates of quadrotor are directly controllable through the fast response onboard controller, due to the small rotational inertia and high torque features of quadrotors \cite{hehn2015real}. 

Similar to \cite{zhou2014vector}, the dynamics in (\ref{eqn:quaddyn}) is differentially flat with the flat output chosen as $\eta = [r^T, \psi^T]^T$. The full state $q=[r^T, \dot{r}^T, \theta, \phi, \psi]^T$ and control $u=[f, \omega^T]^T$ can be represented as an algebraic function of $[\eta^T, \dot{\eta}^T, \ddot{\eta}^T, \dddot{\eta}^T]$.  With the differential flatness property, quadrotor trajectory planning can be simplified as smooth parametric curves. Given a desired trajectory $\eta_d(t) \in C^3$ that is three times differentiable, the feed forward control $u_{FF} =[f_{FF}, \omega_{FF}^T] $ can be derived by inverting the dynamics in (\ref{eqn:quaddyn}), 
\begin{equation*}
  \begin{cases}
    f_{FF} &= \quad -m \|\ddot{r}_d-gz_w\|,  \\
    \omega_{FF} &= \quad \begin{bmatrix} 1 & 0 & -s\theta_d \\
0 & c\phi_d & s\phi_d c\theta_d  \\ 0&-s\phi_d & c\phi_d c\theta_d
\end{bmatrix}\begin{bmatrix} \dot{\phi}_d \\ \dot{\theta}_d \\ \dot{\psi}_d \end{bmatrix}
  \end{cases}
\end{equation*}
where $\theta_d = \text{atan2}(\beta_a, \beta_b)$, $\phi_d=\text{atan2}(\beta_c, \sqrt{\beta_a^2 +\beta_b^2})$, $\beta_a= -\ddot{x}_d\cos\psi_d -\ddot{y}_d\sin\psi_d$, $\beta_b=-\ddot{z}_d+g$, and $\beta_c=-\ddot{x}_d\sin\psi_d+\ddot{y}_d\cos\psi_d$.

Differential flatness only gives the feed forward control $u_{FF}$. In addition, the unknown model error and tracking error need to be handled by a feedback control $u_{FB}$. The actual control applied to the quadrotor is $u=u_{FF}+u_{FB}$, where
\begin{equation*}\label{eqn:fb}
  \begin{cases}
    f_{FB} &= \quad K_p<Rz_w,r_d-r>+K_d<Rz_w,\dot{r}_d-\dot{r}>,  \\
    \omega_{FB} &= \quad K_p\begin{bmatrix} {\phi}_d-{\phi} \\ {\theta}_d-{\theta} \\ {\psi}_d-{\psi} \end{bmatrix} + K_d\begin{bmatrix} \dot{\phi}_d-\dot{\phi} \\ \dot{\theta}_d-\dot{\theta} \\ \dot{\psi}_d-\dot{\psi} \end{bmatrix} + \bar{K}_p\begin{bmatrix} y_d-y \\ x-x_d\\0\end{bmatrix}
  \end{cases}
\end{equation*}
Note that with an inaccurate model, a high-gain feedback controller is needed to counteract both the model error and disturbances. As a better model is learned over time, only a low-gain feedback controller is needed with an improved tracking performance \cite{nguyen2009local}.



\subsection{Learning based Control Using Gaussian Process }
The previous section deals with precise quadrotor models. But it is often difficult to acquire accurate parameters for quadrotor systems. In addition, the model (\ref{eqn:quaddyn}) neglects the uncertain effects of damping, drag force, and wind disturbances. Here, we will use GP models to learn the unmodeled dynamics. The unmodeled dynamics can be captured with six GPs along each dimension in the state space, i.e.,
\begin{eqnarray*}
  \begin{cases}
\ddot{r} \quad\;\;\; = \, gz_w + \frac{1}{m}Rz_wf_z + \begin{bmatrix} \mathcal{GP}_1(0,k(q,q')) \\ \mathcal{GP}_2(0,k(q,q')) \\ \mathcal{GP}_3(0,k(q,q')) \end{bmatrix}, \nonumber \\
\begin{bmatrix} \dot{\phi}\\ \dot{\theta} \\ \dot{\psi}
\end{bmatrix} =  \begin{bmatrix} 1 & s\phi t\theta & c\phi t\theta \\
0 & c\phi & -s\phi  \\ 0&s\phi sc\theta & c\phi sc\theta
\end{bmatrix} \omega  + \begin{bmatrix} \mathcal{GP}_4(0,k(q,q')) \\ \mathcal{GP}_5(0,k(q,q')) \\ \mathcal{GP}_6(0,k(q,q')) \end{bmatrix}, \hspace{0.2in} \nonumber
  \end{cases}
\end{eqnarray*}
where the input to the GPs is $q=[r^T, \dot{r}^T, \theta, \phi, \psi]^T$, and the observations for the GPs are $s = [\ddot{r}^T, \dot{\phi}, \dot{\theta}, \dot{\psi}]^T$, respectively. At a new query point $q_*$, the mean $m_i(q_*)$ and variance $\sigma_i^2(q_*)$ of the unknown dynamics can be inferred with (\ref{eqn:GPinfer}).
Based on the learned dynamics, a differential flatness based feed forward controller  can be derived as,
\begin{equation*}
  \begin{cases}
    f_{FF} &= \quad -m \|\ddot{r}_d-[m_1(q), m_2(q), m_3(q)]^T-gz_w\|,  \\
    \omega_{FF} &= \quad \begin{bmatrix} 1 & 0 & -s\theta_d \\
0 & c\phi_d & s\phi_d c\theta_d  \\ 0&-s\phi_d & c\phi_d c\theta_d
\end{bmatrix}\begin{bmatrix} \dot{\phi}_d-m_4(q) \\ \dot{\theta}_d-m_5(q) \\ \dot{\psi}_d -m_6(q) \end{bmatrix},
  \end{cases}
\end{equation*}
where $\theta_d = \text{atan2}(\bar{\beta}_a, \bar{\beta}_b)$, $\phi_d=\text{atan2}(\bar{\beta}_c, \sqrt{\bar{\beta}_a^2 +\bar{\beta}_b^2})$, $\bar{\beta}_a= -(\ddot{x}_d-m_1(q))\cos\psi_d -(\ddot{y}_d-m_2(q))\sin\psi_d$, $\bar{\beta}_b=-(\ddot{z}_d-m_3(q))+g$, and $\bar{\beta}_c=-(\ddot{x}_d-m_1(q))\sin\psi_d+(\ddot{y}_d-m_2(q))\cos\psi_d$.

\subsection{Recursive Online GP Learning} 
One issue with the GP regression is that the time complexity of GP inference is $O(N^3)$, where $N$ is the number of data points. The majority of the time is used to compute the inverse of the kernel matrix $K$. While various approximation methods can be used to reduce the GP inference time, it is still challenging to perform online GP inference for complex dynamically systems like quadrotor. Here, we propose a recursive online GP Learning method to compute the exact GP inference.

As the quadrotor moves forward, we will actively add multiple relevant data points into the kernel matrix at each time step. At the same time, the data points that contribute the least to the inference are deleted. The recursive data addition and deletion operations are described as following.
\subsubsection{Adding Multiple New Data to the Kernel Matrix}
Let the kernel matrix at the $i$th time step be $K_i$, we can save the matrix inverse result from the previous step as $L_i = (K_i+\sigma_n^2I)^{-1}$. Denote the number of new data to be added as $M$.

With the new data $y_{i+1}$ and kernal vector $k_{i+1}$, we have
\begin{eqnarray}
L_{i+1} &=& \begin{bmatrix} L_i^{-1} & k_{i+1} \\ k_{i+1}^T & c_{i+1}+\sigma_n^2I \end{bmatrix}^{-1}  \nonumber\\
        &=& \left[\begin{matrix} L_i+L_ik_{i+1}(c_{i+1}+\sigma_n^2I-k^T_{i+1}L_ik_{i+1})^{-1}k^T_{i+1}L_i  \\ -(c_{i+1}+\sigma_n^2I-k^T_{i+1}L_ik_{i+1})^{-1}k^T_{i+1}L_i \end{matrix} \right. \nonumber\\
        && \qquad\qquad \left. \begin{matrix}  L_ik_{i+1}(c_{i+1}+\sigma_n^2I-k^T_{i+1}L_ik_{i+1})^{-1} \\  (c_{i+1}+\sigma_n^2I-k^T_{i+1}L_ik_{i+1}) \end{matrix} \right]. \nonumber
\end{eqnarray}
Notice that inversion operation only needs to be performed on a $M\times M$ matrix rather than a large $N\times N$ matrix.
\subsubsection{Deleting Multiple Old Data from the Kernel Matrix}
After deleting $M$ data points from the old Kernel matrix inversion $L_i = (K_i+\sigma_n^2I)^{-1}$, the new inverse of the kernel matrix becomes $\bar{L}_i = (\bar{K}_i+\sigma_n^2I)^{-1}$.

First, the data to be deleted is permuted to the bottom of the kernel matrix with a permutation matrix $P_{\pi}$, where $\pi:\mathbb{N}\to \mathbb{N}$ is a permutation of $N$ elements. The permuted kernel matrix is 
$K_i^p = P_{\pi}K_iP_{\pi}^T$, which can be written into a block matrix form, 
\begin{equation*}\label{eqn:kpblock}
K_i^P = \begin{bmatrix} \bar{K}_{i} & E_{i} \\ E_{i}^T & F_{i} \end{bmatrix},
\end{equation*}
where $E_i, F_i$ are the known parts to be deleted. Similarly, 
\begin{eqnarray}
L_i^P &=& P_{\pi}L_iP_{\pi}^T \nonumber \\
            &=& \begin{bmatrix} \bar{L}_i^{-1} & E_{i} \\ E_{i}^T & F_{i}+\sigma_n^2I \end{bmatrix}^{-1}. \nonumber 
\end{eqnarray}
Since $L_i^P$ is known, it can be written into block matrix form with the same block dimensions with (\ref{eqn:kpblock}), 
\begin{equation*}
L_i^P = \begin{bmatrix} A_{i} & B_{i} \\ B_{i}^T & C_{i} \end{bmatrix}.
\end{equation*}
With the block matrix inversion rule, $\bar{L}_i$ can be recovered as
\begin{equation*}\label{eqn:blockrecursion}
\bar{L}_i = A_i-B_iC_i^{-1}B_i^T,
\end{equation*}
which means to perform the deletion operation, the only matrix inverse required is $C_i^{-1}\in\mathbb{R}^{M\times M}$. 

With the recursive data addition and deletion method, the GP inference can be obtained efficiently online.

\section{Simulation Results}\label{sec:sim}
The GP based learning algorithm is validated on a simulated quadrotor model in two examples, i.e., online learning of quadrotor dynamics and learning safety barrier certificates. In the simulation, the actual weight of the quadrotor is 1.4 times the weight used in the computation. In addition, an unknown constant wind of $0.1g$ is applied in the environment as illustrated in Fig. \ref{fig:simquad}. Since the standard fixed pitch quadrotor cannot generate reverse thrust, the thrust control is limited to $f_z\in [-1.8mg,0]$. This simulation setup is very challenging, because the learning based quadrotor controller needs to deal with very inaccurate model and limited thrust. 
\begin{figure}[h]
  \centering
  \includegraphics[width=0.8\linewidth]{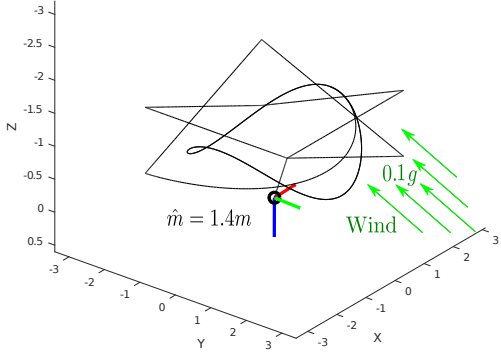}
  \captionof{figure}{A simulated quadrotor flies in an unknown wind field with an inaccurate model.}
\label{fig:simquad}
\end{figure}

\subsection{Online Learning of Quadrotor Dynamics}
 In the first example, the quadrotor is commanded to track a nominal
trajectory (illustrated in Fig. \ref{fig:simquad}) using a differential flatness based controller with the given inaccurate model. A PD controller is wrapped around to stabilize the quadrotor. During the simulation, the quadrotor is intentionally pushed to unknown regions that has not been explored before. This will help us evaluate the scalability of the algorithm. 

The desired trajectory of the quadrotor is given as $\hat{\eta}=[\hat{r}(t)^T, \hat{\psi}(t)]\in C^3$, while the actual trajectory is ${\eta}=[{r}(t)^T, {\psi}(t)]$. In practice, the actual trajectory might deviate significantly from the desired trajectory when the model is very inaccurate. To track the desired trajectory, the nominal trajectory is designed with a pole placement controller,
\begin{equation*}\label{eqn:poleCLF}
\dddot{r}_i = \dddot{\hat{r}}_i - K\cdot [(r_i-\hat{r}_i), ~(\dot{r}_i-\dot{\hat{r}}_i), ~(\ddot{r}_i-\ddot{\hat{r}}_i)]^T.
\end{equation*}
In the simulation, the sample size of the recursive GP model is fixed at 300 data points. At each time step, the most irrelevant data point is thrown away, and the most relevant data point is added to the GP model. The data relevance is decided by the kernel function $k(q,q^*)$, where $q=[r^T, \dot{r}^T, \theta, \phi, \psi]^T$. It can observed that the tracking error of the learning based controller is significantly smaller than the tracking error without GP inference, as shown in Fig. \ref{fig:gpcompare}. 
\begin{figure}[h]
\centering
\begin{subfigure}{.24\textwidth}
  \centering
  \includegraphics[width=1\linewidth]{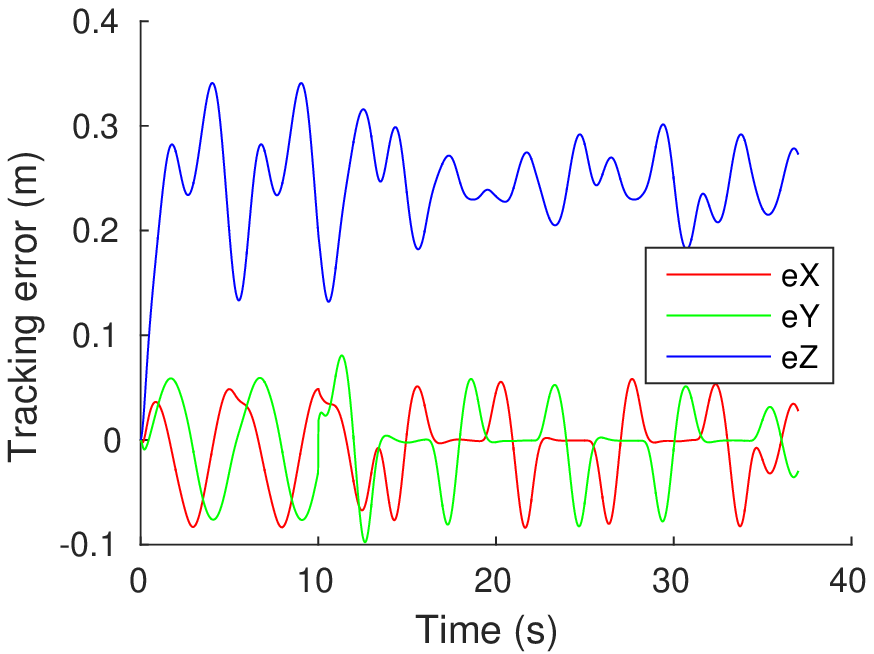}
  \captionof{figure}{Tracking error without GP}
  \label{fig:gp1time}
\end{subfigure}
\centering
\begin{subfigure}{.23\textwidth}
  \centering
  \includegraphics[width=1.0\linewidth]{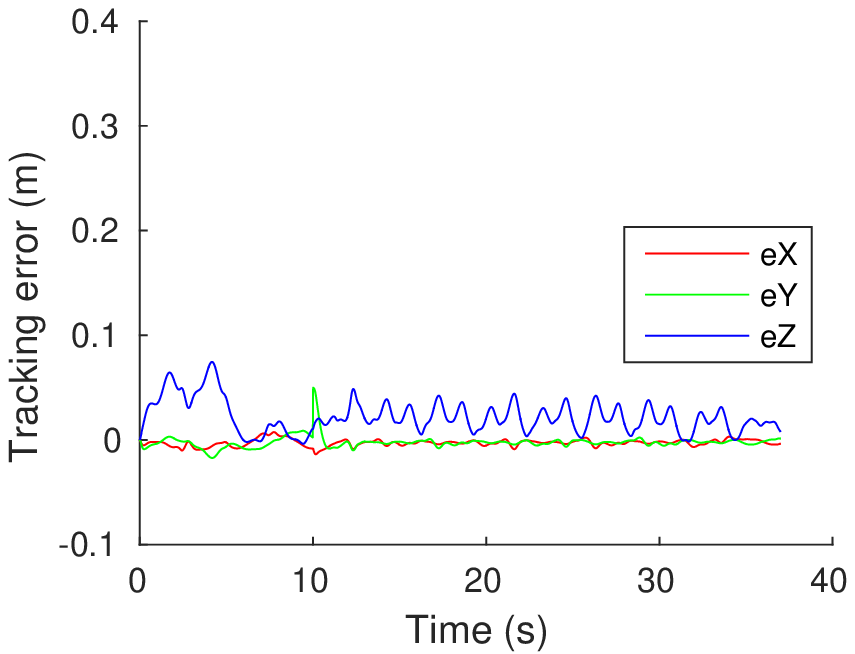}
  \captionof{figure}{Tracking error with GP}
  \label{fig:gp0error}
\end{subfigure}%
\caption{Tracking error of the differential flatness based flight controller with and without GP inference.}
\label{fig:gpcompare}
\end{figure} 

With the recursive learning strategy, it is demonstrated in Fig. \ref{fig:gp1error} that the GP inference time is always kept below $20ms$. Thus, the recursive GP inference method is very suitable for online learning of quadrotor dynamics. 
\begin{figure}[h]
  \centering
  \includegraphics[width=0.7\linewidth]{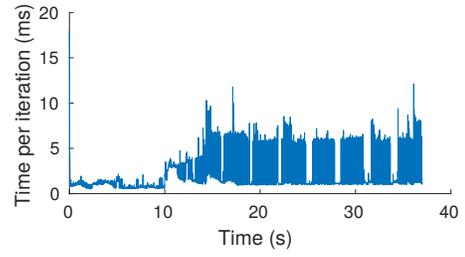}
  \captionof{figure}{Recursive GP inference time per iteration.}
  \label{fig:gp1error}
\end{figure}

By pushing the quadrotor to unexplored regions, we can found that learning with $q'=[\dot{r}^T, \theta, \phi, \psi]^T$ yields much better scalability than learning with $q=[r^T, \dot{r}^T, \theta, \phi, \psi]^T$. The reason might be the position $r$ is not as important as other features in the current simulation setup.

\subsection{Learning Safety Barrier Certificates} \label{eqn:learnCBF}
In this example, the motion of the quadrotor is constrained within an ellipsoid safe region, i.e., $$ \frac{x^2}{0.16} + \frac{y^2}{0.16} + \frac{(z+0.8)^2}{0.36} \leq 1.$$  The quadrotor is controlled to fly back and forth on a vertical path inside the ellipsoid. The goal is to learn how aggressively the quadrotor can fly in the $z$ direction with an inaccurate model and limited thrust. 

The barrier certificates are parameterized as
\begin{eqnarray*}
h_{\mu}(r) &=& 1-\frac{(z+0.8)^2}{0.36} -\mu\dot{z}^2 \\ &&- \frac{x^2}{0.16}- \frac{y^2}{0.16} - \frac{\dot{x}^2}{0.25}- \frac{\dot{y}^2}{0.25}\geq 0,
\end{eqnarray*}
where $\mu$ is the barrier parameter to regulate how fast the quadrotor can fly in the $z$ direction. Small values of $\mu$ correspond to large admissible speed $\dot{z}$, which means more aggressive flight behavior. Thus, the objective of the learning process is to minimize $\mu$ with the collected data.

To reduce the number of required sample points, the adaptive sampling strategy developed in Section \ref{sec:adaptsample} was adopted. An illustrative example of the adaptive sampling strategy is given in Fig. \ref{fig:simadapt}. It can be observed that places closer to the boundary of the safe region ($z=-1.2$ and $z=-0.2$) are sampled much denser than the place closer to the center of the safe region ($z=0$). Furthermore, downward speed ($\dot{z}>0$) is sampled much denser than the upward speed ($\dot{z}<0$). This might be caused by the lack of reverse thrust to counter the unmodeled dynamics.
\begin{figure}[h]
  \centering
  \includegraphics[width=0.95\linewidth]{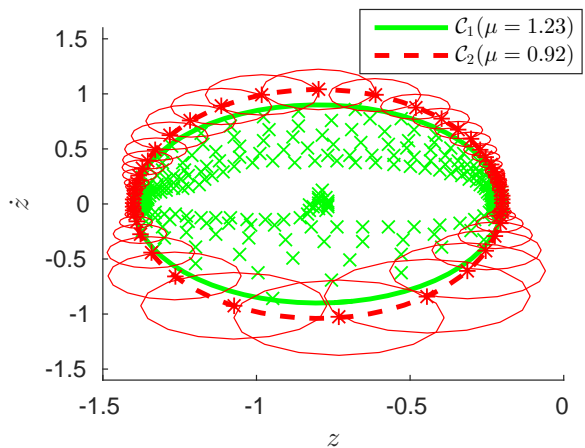}
  \captionof{figure}{Adaptive sampling of the state space. The region enclosed by the solid green ellipse $\mathcal{C}_1$ is the current safe region, while the region enclosed by the dashed red ellipse $\mathcal{C}_2$ is the optimized next safe region. The green cross markers and red asterisk markers are the data points already sampled and to be sampled, respectively. The red circles centered at those sample points are the confident safe regions based on (\ref{eqn:xktauCondition}). All the unexplored region between $\mathcal{C}_1$ and $\mathcal{C}_2$ are covered by the circular confident safe region.}
  \label{fig:simadapt}
\end{figure}

A conservative barrier certificate ($\mu=6.3$) is provided at the beginning of the learning process. Then, the quadrotor gradually explores the safe region $\mathcal{C}_0$ and expands it to $\mathcal{C}_n$ ($\mu=0.6$), as illustrated in Fig. \ref{fig:simbarrierIF}. The nominal exploration controller is always regulated by the barrier certificates using the QP-based controller in (\ref{eqn:QPcontroller}). During the learning process, the quadrotor never leaves the barrier certified safe region.
\begin{figure}[h]
  \centering
  \includegraphics[width=0.95\linewidth]{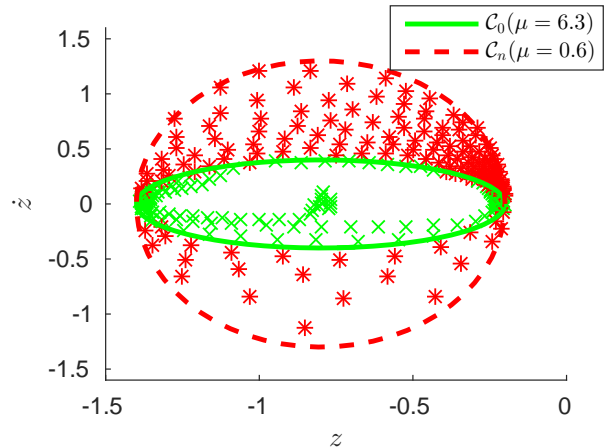}
  \captionof{figure}{Initial and final barrier certificates. The regions enclosed by the solid green ellipse ($\mathcal{C}_0$) and dashed red  ellipse ($\mathcal{C}_n$) are the initial and final barrier certified safe regions, respectively. The green cross markers and red asterisk makers are the sampled data points.}
  \label{fig:simbarrierIF}
\end{figure}


\section{Conclusions}\label{sec:conclusions}
A safe learning algorithm based on barrier certificates was developed in this paper. The learning controller is regulated by the barrier certificates, such that the system never enters the unsafe region. The unmodel dynamics of the system was approximated with a Gaussian Process, from which a high probability safety guarantee for the dynamical system was derived.  The barrier certified safe region is gradually expanded as the uncertainty of the system dynamics is reduced with more data. This safe learning technique was applied on a quadrotor system with 3D nonlinear dynamics. The computation time of this learning method is reduced significantly with an adaptive sampling strategy and a recursive GP inference method. Simulation results demonstrated the effectiveness of the proposed method.


\addtolength{\textheight}{-12cm}   


\bibliographystyle{abbrv}
\bibliography{mybib}
\end{document}